\newif\ifreport
\reporttrue

\RequirePackage{amsmath}
\documentclass[runningheads]{llncs}

\usepackage{xcolor}
\usepackage{marginnote}
\usepackage{comment}
\usepackage[utf8]{inputenc}
\usepackage[T1]{fontenc}
\usepackage{figlatex}
\usepackage{paralist}
\usepackage{xspace}
\usepackage{url}
\usepackage[american]{babel}
\usepackage{hyperref}
\usepackage{graphicx}
\graphicspath{{fig/}}
\usepackage{pgfplots}
\pgfplotsset{compat=1.18}
\usepackage{cleveref}
\usepackage{import}
\RequirePackage{tikz}
\usepackage{pgf}
\usepackage{graphicx}
\usetikzlibrary{shadows}
\usetikzlibrary{decorations.pathreplacing}
\usetikzlibrary{positioning, arrows.meta}
\usetikzlibrary{shapes.geometric, calc}
\usetikzlibrary{patterns}
\usepackage{amsfonts}
\usepackage{amsmath}
\usepackage{mathtools} 
\usepackage{amssymb}
\usepackage{silence}

\usepackage{subcaption}
\usepackage{physics}      %
\usepackage{dsfont}

\DeclareMathOperator{\CR}{\textsc Cr}

\newcommand{\LRU}{\textsc{LRU}\xspace}

\newcommand{\OPT}{\textsc{Opt}\xspace}
\newcommand{\LLRU}{\textsc{LLRU}\xspace}
\newcommand{\MARKER}{\textsc{Marking}\xspace}
\newcommand{\OPTDIST}{\textsc{Opt-Dist}\xspace}
\newcommand{\LRUDIST}{\textsc{LRU-Dist}\xspace}

\newcommand{\LLRUDIST}{\textsc{LLRU-Dist}\xspace}
\newcommand{\MARKERDIST}{\textsc{Marker-Dist}\xspace}

\newif\ifnextstep
\nextstepfalse

\renewcommand{\geq}{\geqslant}
\renewcommand{\leq}{\leqslant}
\renewcommand{\epsilon}{\varepsilon}

\usepackage[disable]{todonotes}

\newcommand{\ao}[1]{\todo[author=Adrien,inline,color=red!50]{ \footnotesize{#1}}\xspace}

\newcommand{\figspace}{\vspace{-5pt}}
\ifreport
    \renewcommand{\figspace}{}
\fi

\begin{document}
\title{Cache Management for Mixture-of-Experts LLMs
  \ifreport
-- extended version
  \fi
}
\author{Spyros Angelopoulos\inst{1}\orcidID{0000-0001-9819-9158} \and\\
  Loris Marchal\inst{1}\orcidID{0000-0002-5519-9913} \and\\
  Adrien Obrecht\inst{2}\orcidID{0009-0007-6037-9787} \and\\
  Bertrand Simon\inst{3}\orcidID{0000-0002-2565-1163}}
\authorrunning{S. Angelopoulos, L. Marchal, A. Obrecht and B. Simon}
\institute{CNRS, International Laboratory on Learning Systems,  Montreal, Canada \and
  École Normale Supérieure de Lyon, France \and
    CNRS IN2P3 Computing Center, Lyon, France \\
 \email{adrien.obrecht@ens-lyon.fr}
 \email{[spyros.angelopoulos,loris.marchal,bertrand.simon]@cnrs.fr}}%
\maketitle              %
\begin{abstract}
  Large language models (LLMs) have demonstrated remarkable capabilities across a variety of tasks. One of the main challenges towards the successful deployment of LLMs is memory management, since they typically involve billions of parameters. To this end, architectures based on Mixture-of-Experts have been proposed, which aim to reduce the size of the parameters that are activated when producing a token. This raises the equally critical issue of efficiently managing the limited {\em cache} of the system, in that frequently used experts should be stored in the fast cache rather than in the slower secondary memory. 
  
  In this work, we introduce and study a new {\em paging} problem that models expert management optimization. Our formulation captures both the layered architecture of LLMs and the requirement that experts are cached efficiently. We first present lower bounds on the {\em competitive ratio} of both deterministic and randomized algorithms, which show that under mild assumptions, LRU-like policies have good theoretical competitive performance. We then propose a layer-based extension of LRU that is tailored to the problem at hand. %
  Extensive simulations on both synthetic datasets and actual traces of MoE usage show that our algorithm outperforms policies for the classic paging problem, such as the standard LRU.

\end{abstract}
\section{Introduction}

Large Language Models (LLMs) have revolutionized the application of AI in fields as diverse 
as text generation, machine translation, and natural language
understanding. LLMs such as GPT-4, PaLM and LLaMA are at the core of applications ranging from
conversational agents to code generation tools, bringing unprecedented
fluency and accuracy to text production. The performance of LLMs is
closely tied to their {\em size}, with modern models consisting of
billions or even trillions of parameters~\cite{minaee2024large}. While
this scale greatly enhances performance, it also poses a major challenge:
the management of the immense computational resources required for their training and usage.

Once trained, LLMs are deployed for {\em inference}, the process of generating text or predictions, which typically needs to be rapid for real-time applications. However, inference is computationally intensive, and thus particularly demanding in resource-constrained environments such as mobile devices or systems with limited memory and processing power. This challenge has led to extensive research into optimization methods that enhance efficiency without significantly compromising model accuracy. Notable approaches include model compression techniques such as quantization~\cite{nagel2021white}, which reduces the precision of model parameters (commonly known as {\em weights}), and sparsification~\cite{hoefler2021sparsity}, which prunes less important parameters.

One well-established strategy for model sparsification is the Mixture of Experts (MoE)
architecture. We provide a high-level overview of this model, and we refer to~\cite{du2022glam,vaswani2017attention} for a more detailed technical definition. A MoE model consists of a sequence of
{\em transformer layers} (or layers, for simplicity). Each layer is composed of an {\em attention} mechanism,  which selectively focuses on the most relevant parts of the input, and a MoE block, as depicted in Figure~\ref{fig.schema-moe}.
A {\em gating} function uses the attention output to determine which experts in the MoE block are activated, while the others remain inactive.
This approach offers two key advantages: 
first, it achieves high accuracy in text generation tasks because the model can leverage a large pool of weights across all the experts. Second, since only a small proportion of experts are active during each inference, it significantly reduces storage and computational costs.

\begin{figure}[bt]
  \centering
  \includegraphics[width=0.55\linewidth]{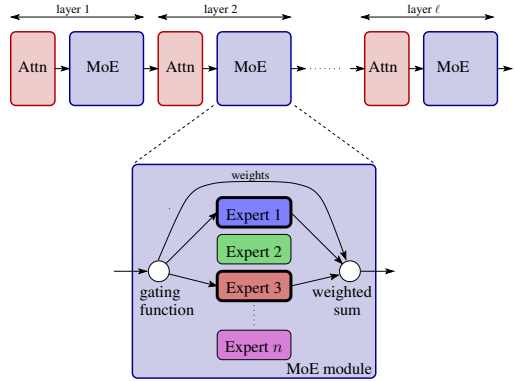}
  \caption{High-level illustration of the Mixture-of-Experts architecture}
  \label{fig.schema-moe}
  \figspace
\end{figure}

While MoE models allow for a selective, and parsimonious activation of experts, their management remains challenging. First, in order to maintain high accuracy, the combined size of all the experts’ weights is generally larger than the total size of weights in non-MoE models~\cite{minaee2024large}, leading to storage complications. LLMs are autoregressive models, meaning that the entire pipeline of {\em all} layers  has to be processed for the production of each token—a word or a portion of a word in the generated text. Hence, even though only a limited number of experts are activated per layer for each token, loading their weights into memory incurs a significant cost that adds to the overall computational cost. Such issues were observed in~\cite{mazur23}, which showed that caching certain experts weights in the memory, instead of reloading them for each layer, improved the inference time. This was achieved using a simple Least-Recently-Used (LRU) replacement policy.  However, the following question arises naturally:

{\em What are the power and limitations of caching policies for MoE models of LLMs? What theoretical and experimental improvements can one obtain, especially relative to standard caching approaches suitable for much simpler environments, such as LRU?}

\subsection{Contribution}
\label{subsec:contribution}

In this work, we answer the above question by introducing and studying a new model of {\em online caching} that is tailored to the characteristics of MoE-based LLMs. The model allows us to provide rigorous theoretical bounds on the performance of both deterministic and randomized caching policies, using the framework of {\em competitive} analysis~\cite{borodin1998online}. This framework has been infuential in the study of many other variants of caching (see Section~\ref{subsec:related} for related works) and we likewise demonstrate, for the first time, that it can be equally useful in the context of LLMs.

In Section~\ref{sec.model-RW} we describe in detail a {\em layered} model of caching, which defines the interactions between the experts across the layers, and involves parameters such as the number of layers $\ell$, the number of experts $n$, and the cache size $k$. In Section~\ref{sec:competitive} we give lower bounds on the competitive ratio of deterministic and randomized caching policies. These bounds show that as long as the parameters $n$ or $\ell$ are constant (which is typically the case in all current models), the best-possible competitive ratios cannot improve, asymptotically, upon standard policies such as LRU. 

While the theoretical lower bounds may seem, at first sight, quite restrictive, they provide insights about how to obtain algorithms that perform better in practice. In Section~\ref{sec:llru} we propose an extension of LRU that is adapted to the problem at hand. 
In Section~\ref{sec:experiments} we perform an extensive experimental evaluation of our algorithm on a variety of traces from real data of MoE usage, as well as on synthetic datasets. The results demonstrate that our algorithm clearly outperforms standard LRU, and showcases the benefits of a layered-based approach in the deployment of efficient cache policies.

\subsection{Related Work}
\label{subsec:related}

Online caching, also known as {\em paging}, is one of the fundamental online optimization problems, and has served as proving grounds for the introduction and the development of competitive analysis. In their seminal work, Sleator and Tarjan~\cite{sleator1985amortized} considered the standard model with a cache size $k$, and showed that an optimal deterministic competitive ratio equal to $k$ can be achieved by a variety of {\em marking} policies, to which LRU belongs. Randomization can help improve the competitive ratio to a tight bound $H_k=\Theta(\log k)$, where $H_n$ is the $n$-th harmonic number~\cite{fiat_competitive_2002}.%
\ifreport
Paging has been studied in a multitude of settings, such as weighted requests~\cite{chrobak1991weighted}, hierarchical cache systems~\cite{chrobak2004online}, reordering cache models~\cite{karlin2000combining}, and energy-efficient policies~\cite{agrawal2011green}. We refer to Chapters 3 and 4 in~\cite{borodin1998online} as well as to the more recent survey by Karlin and Koutsoupias~\cite{karlin2020beyond}.

Caching is a critical aspect in optimizing the inference efficiency of LLMs. A common approach involves the use of {\em Key-Value} (KV) caching~\cite{li2024survey}, which stores past activations from the attention mechanism to prevent redundant computations during the  autoregressive text generation.
Our work, in contrast, addresses a new caching setting that is specific to MoE models within LLMs. Unlike KV caching, which focuses on reusing intermediate attention outputs, we investigate caching policies aimed at efficiently managing the dynamic loading and unloading of expert weights across layers. This involves challenges due to the selective activation of experts and the substantial memory footprint associated with MoE architectures~\cite{mazur23}, making our approach complementary to existing KV caching solutions.
\fi

\section{Layered Paging: a Model for MoE Caching in LLMs}
\label{sec.model-RW}

In this section, we formulate the problem we will study, and which serves as a model of LLM caching. We first discuss, in Section~\ref{subsec:moe.caching}, the pertinent structure and properties of MoE architectures, then in Section~\ref{subsec:layered} we define the problem of layered paging.

\subsection{Caching in MoE Architectures}
\label{subsec:moe.caching}

Text generation with LLMs involves producing a sequence of {\em tokens}, which are units of text that can represent words, subwords, or punctuation marks. To generate a single token, the model processes information through $\ell$ layers sequentially. Once the last layer is processed, the token is produced and then fed back into the first layer to begin the computation for the next token. This cycle repeats until an “End of Text” token is generated or the maximum number of tokens is reached. In modern LLMs, the number of layers $\ell$ typically ranges from 10 to 100, depending on the model's size and complexity. In the MoE paradigm, in particular, each layer involves $n$ experts, which are specialized neural network modules that enhance model efficiency and performance. In current models, $n$ generally varies between 8 and 256. We denote by $E_i^{(j)}$ the $i$th expert in layer $j$, where $i \in [1,n]$ and $j \in [1,\ell]$.

In a MoE module, only a few experts, among the $n$ total, are typically required. For simplicity, we consider here that a single expert per
layer is requested. That is, in order to produce a single token, the required (ordered) series of experts is of the form  $E_{i_1}^{(1)},E_{i_2}^{(2)},\ldots,E_{i_\ell}^{(\ell)}$, for some $i_1,...,i_\ell$. Hence, the total series of experts that we have to process to produce all tokens is of the form

$$ E_{i_1^1}^{(1)}, E_{i_2^1}^{(2)}, \ldots, E_{i_\ell^1}^{(\ell)}, E_{i_1^2}^{(1)},
E_{i_2^2}^{(2)}, \ldots, E_{i_\ell^2}^{(\ell)}, \ldots $$

Computing an expert requires its data (i.e., the weights of the neural network) to be loaded in the memory. Thus, to speed up computation, it is important to ensure that frequently reused expert data are kept in the fast cache memory which is of limited size $k$ (we assume that all experts produce data of the same size).  Figure~\ref{fig.example} provides an illustration.

\begin{figure}[t!]
  \centering
  \scalebox{0.55}{%

\colorlet{c1}{red}
\colorlet{c2}{blue}
\colorlet{c3}{green!40!black}
\pgfdeclarelayer{background}
\pgfsetlayers{background,main}

\begin{tikzpicture}[scale=1.6]

\foreach \layer in {1,2,3,4}
{
    \node at (1.3*\layer, -0.3) (L\layer) {\large $L_{\layer}$};
}

\foreach \layer in {1,2,3,4}
{
    \foreach \expert in {1,2,3,4}
    {
        \node[draw,very thick, circle, minimum size=1cm, fill=white] at
            (1.3*\layer, -\expert) (E\layer\expert) {$E_{\expert}^{(\layer)}$};
    }
    \draw[dotted, thick] (1.3*\layer - 0.5, -0.5) rectangle ++(1, -4);
}

\draw[decorate, decoration={brace, amplitude=10pt}] (0.7, -4.5) -- (0.7, -0.5)
node[midway,xshift=-8pt,left]{\rotatebox{90}{Experts per layer ($n=4$)}};
\draw[decorate, decoration={brace, amplitude=10pt}] (5.7, -4.6) -- (0.8, -4.6)
node[midway,yshift=-16pt] {Number of layers $(\ell=4)$};

\begin{pgfonlayer}{background}
    \draw[c1!60, opacity=0.5, line width=13mm, rounded corners, line cap=round]
        (E11.center) -- (E21.center) -- (E32.center) -- (E43.center); 
    \draw[c2!60, opacity=0.5, line width=13mm, rounded corners, line cap=round]
        (E12.center) -- (E22.center) -- (E32.center) -- (E41.center);
    \draw[c3!60, opacity=0.5, line width=13mm, rounded corners, line cap=round] 
        (E12.center) -- (E23.center) -- (E34.center) -- (E42.center);
\end{pgfonlayer}

\coordinate (sequence) at (6, 0.8);
\node at (sequence) {\large $S = (
        \overbrace{E_1^{(1)}, E_1^{(2)}, E_2^{(3)},
        E_3^{(4)}}^{\textcolor{c1}{\text{Token 1}}}, 
        \overbrace{E_2^{(1)}, E_2^{(2)}, E_2^{(3)},
        E_1^{(4)}}^{\textcolor{c2}{\text{Token 2}}}, 
        \overbrace{E_2^{(1)}, E_3^{(2)}, E_4^{(3)},
        E_2^{(4)}}^{\textcolor{c3}{\text{Token 3}}}
    )$};

\path[<-, dashed, thick, bend left, gray] 
    ($(sequence) + (0.4,-0.4)$) edge ($(sequence) + (-1.7,-0.4)$);
\path[<-, dashed, thick, bend left, gray] 
    ($(sequence) + (1.5,-0.4)$) edge ($(sequence) + (-0.5,-0.4)$);
\node[gray] at ($(sequence) - (0, 0.8)$) {reused};

\node at ($(sequence) + (2.8, -0.8)$) (eviction) {\textcolor{red}{First eviction}};
\draw[->, thick, red] (eviction) -- ($(sequence) + (2.8,-0.4)$);

\coordinate (memory) at (6.5, -2.5);
\pgfmathsetmacro\memscale{0.7}

\pgfgettransformentries{\xscale}{\o}{\o}{\yscale}{\o}{\o}
\tikzstyle{memcell}=[rectangle, minimum height=\memscale*\yscale cm, minimum
width=\memscale*\xscale cm, draw=black, thick, inner sep=0pt];

\node[memcell, fill=c1!40] at ($(memory) + (0*\memscale, 0)$) (mem1) {$E_1^{(1)}$};
\node[memcell, fill=c1!40] at ($(memory) + (1*\memscale, 0)$) (mem2) {$E_1^{(2)}$};
\node[memcell, fill=c1!40] at ($(memory) + (2*\memscale, 0)$) (mem3) {$E_2^{(3)}$};
\node[memcell, fill=c1!40] at ($(memory) + (3*\memscale, 0)$) (mem4) {$E_3^{(4)}$};
\node[memcell, fill=c2!40] at ($(memory) + (4*\memscale, 0)$) (mem5) {$E_2^{(1)}$};
\node[memcell, fill=c2!40] at ($(memory) + (5*\memscale, 0)$) (mem6) {$E_2^{(2)}$};
\node[memcell, fill=c2!40] at ($(memory) + (6*\memscale, 0)$) (mem7) {$E_1^{(4)}$};
\node[memcell, fill=c3!40] at ($(memory) + (7*\memscale, 0)$) (mem8) {$E_3^{(2)}$};

\draw[thick, line width=1mm] 
    ($(memory)-(0.5*\memscale,0.5*\memscale)$) rectangle ++(8*\memscale, \memscale);

\node at ($(memory) + (4*\memscale-0.5, 0.7)$) {\large Memory};
\node at ($(memory) + (4*\memscale-0.5, -0.7)$) {Max $k=8$ experts loaded at once};

\end{tikzpicture}
}
  \caption{Example of experts used in the production of three tokens. The
    models has $\ell=4$ layers, each with $n=4$ experts. The experts
    used for each token are represented by a different color. Expert~$E_2^{(3)}$
    is used for token 1 and reused for 2, hence resulting in a cache
    hit. The figure on the right depicts the state of the cache when processing the third layer for token 3: Expert $E_4^{(3)}$ is
    requested, but the memory is already full, hence some other expert
    needs to be evicted.}
  \label{fig.example}
  \figspace
\end{figure}
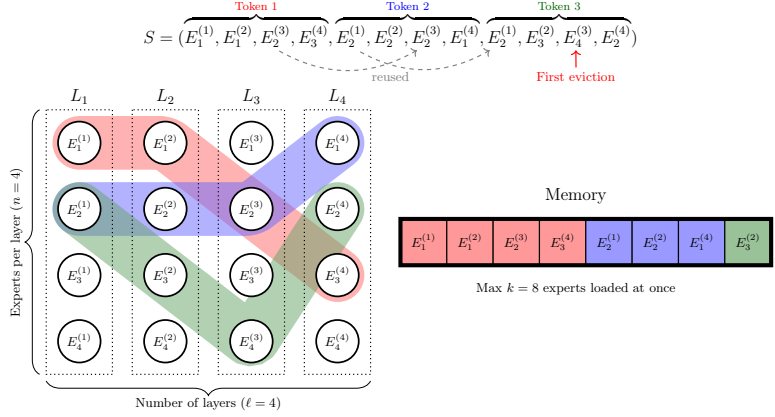

\subsection{Layered Paging}
\label{subsec:layered}

We now proceed with the formal definition of our problem, which captures the setting discussed in the previous section. We say that the data produced by each expert corresponds to a {\em page}, and recall that the cache has a capacity $k$, in that it can store $k$ pages. There are $n\ell$ possible pages in total, given that each of the $n$ experts may produce a page in each of the $\ell$ layers. The set of pages $P$ can be thus decomposed in $\ell$ disjoint sets $L_1, \ldots L_\ell$, each of cardinality $n$. We denote the pages inside each set $L_i$ as $\{p_1^{(i)},p_2^{(i)},\dots,p_n^{(i)}\}$. The input  $\sigma$ to our problem is a sequence of pages $\sigma=p_1,p_2, \ldots$, with the constraint that 
$p_i \in L_{(i-1 \bmod \ell)+1}$, modeling the layered structure of experts in the MoE architecture. For the purpose of analysis, a {\em round} in $\sigma$ is defined as a consecutive block of pages of the form $p_j,\ldots,p_{j+\ell - 1}$ where $p_i$ with $i\in [j,j+\ell-1]$ belongs to layer  $L_{i-j+1}$, so $\ell$ divides $j-1$.  For an illustration, each round is depicted in Figure~\ref{fig.example} by its proper color. 

Having defined the set of pages and the structure of an input sequence, the optimization objective is as in the standard online paging. Namely, if a requested page is present in the cache, which corresponds to a cache {\em hit}, then the page is served at zero cost. In contrast, if a requested page is not in the cache, which corresponds to a cache {\em miss}, then the algorithm must determine which page to evict from the cache, in order to bring in the requested page. In this case, the page is served at cost equal to 1. Note that any paging algorithm can be fully described as a cache eviction policy, which determines the page to be evicted in the event of a cache miss.

We will refer to the above problem as the {\em $\ell$-layer paging} problem, or {\em layered paging}, for simplicity. Note that this problem is a generalization of standard online paging, since the latter is equivalent to  $1$-layered paging. To analyze the performance of algorithms for the problem, we will rely on the canonical framework of competitive analysis. Let $A$ denote an online paging algorithm, and $A(\sigma)$ its {\em cost} on input $\sigma$, i.e., the total number of cache misses it incurs on $\sigma$. Let also \OPT denote an optimal {\em offline} algorithm that has foreknowledge of $\sigma$. It is known that an optimal offline algorithm is the one that, upon each request that results in a cache miss, evicts a page that will be requested the furthest in the future~\cite{belady_study_1966}.

\begin{definition}[Competitive Ratio]
  A deterministic algorithm $A$ is 
  $r$-competitive if there is a constant $c \in
  \mathbb{R}$ such that for any request sequence $\sigma$,  it holds that
  \[
  A(\sigma) \leq r \cdot {\OPT}(\sigma) + c.
  \]
  Similarly, a randomized algorithm $A$ is $r$-competitive if for every request $\sigma$ 
 \[
 \mathds{E}(A(\sigma)) \leq r \cdot {\OPT}(\sigma) + c,
 \]
 where the expectation is over the random choices of $A$. 
 \label{def:comp.ratio}
\end{definition}

We will denote by $\CR(A)$ the competitive ratio of an algorithm $A$.

\ifreport
The paging problem was often studied in the more general setting of the
$k$-server problem \cite{manasse_competitive_1990,young_k-server_1994}. The book \textit{"Randomized Algorithms"} by
\textsc{Motwani} \& \textsc{Raghavan} \cite{motwani_randomized_1995}
conveys a lot of the intuition behind the study of randomized
algorithms, with a section focusing on the paging problem, proving
some previously mentioned results.
\fi

\section{Competitive Analysis of Layered Paging}
\label{sec:competitive}

In this section, we present theoretical results on the competitive ratio of deterministic and randomized algorithms for layered paging.

%

\subsection{Deterministic Algorithms}
\label{sec.det}

A first approach is to consider an online algorithm that devotes a predetermined portion of the cache to requests of each layer; {\itshape e.g.}, a portion of size approximately equal to $k/\ell$ to each layer. While such simple approaches are used in practical implementations~\cite{mazur23}, we show that they result in poor theoretical performance. Fixed cache partitions lead to inefficiencies, aligning with prior results on parallel paging \cite{agrawal_green_2020}.

\begin{theorem}
  \label{thm.det_fixed}
    Any deterministic algorithm that allocates a fixed portion of the cache
    to a requests of a given layer has  unbounded competitive ratio, as long as $n\geq2$ and $\ell\geq 2$.
\end{theorem}

\ifreport
\begin{proof}

We consider a layered paging instance and an algorithm $A$ with a fixed cache allocation per layer. Let the total cache size $k$ be such that $n+\ell-1 \leq k < n\ell$.  Note that there exists at least one layer $L_z$ for which $A$'s cache always holds fewer than $n$ pages, as $k/\ell < n$.

We define the request sequence $\sigma$ such that the $i$-th round contains pages $p_1^{(j)}$ for each $j\neq z$ and page $p_{i \pmod n +1}^{(z)}$.
Here, each layer contributes one page, except $L_z$, which cycles through $n$ pages.

Since $A$ cannot hold $n$ pages of $L_z$, it incurs at least one cache miss every $n$ rounds (or $n\ell$ requests). An optimal algorithm, if $k > n + \ell$, can store all requested pages, ensuring no cache miss after initially filling the cache.

 Thus, $A$ has an unbounded competitive ratio.
 \qed
\end{proof}
\fi

The following theorem provides a lower bound on the competitive ratio of any deterministic algorithm. The result shows that the best competitive ratio cannot be much smaller than that of standard online paging, namely $k$.

\begin{theorem}
  For every value of $\ell$, every deterministic layered paging algorithm has a competitive ratio
  at least $k-\ell+1$.
\end{theorem}

\ifreport
\begin{proof}
Consider a deterministic layered paging algorithm $A$, and let $k=n\ell-1$. We define a sequence $\sigma$ such that two consecutive rounds differ in exactly one request: the page $p_i^{(j)}$ missing from $A$'s cache at the start of the round is requested and the other requests are repeated from the previous round. The first round contains arbitrary valid requests.

The algorithm $A$ suffers one cache miss at every round by construction. Consider a round $i$ where an optimal solution \OPT suffers a cache miss. \OPT evicts the page which is requested the latest. There are $\ell$ pages requested during round $i$, so at most $k$ pages requested between rounds $i$ and $i+k-\ell$, included. The next cache miss of \OPT is at round $i+k-\ell+1$ at the earliest, so \OPT suffers at most one cache miss every $k-\ell+1$ rounds. Therefore, $A(\sigma) \geq (k-\ell+1) \OPT(\sigma)$, so $\CR(A) \geq k-\ell+1$.
\qed
\end{proof}
\fi

\LRU is $k$-competitive for layered paging, since it is $k$-competitive in standard paging, but does it achieve a better competitive ratio for $\ell$-layered paging? The following theorem answers in the negative, for any $\ell$. 

\begin{theorem}
  For any $\ell$ and $k$ such that $\ell$ divides $k+1$, $\CR(\LRU)\geq k$.
\label{thm:lru.tight}  
\end{theorem}

\ifreport
\begin{proof}
Let $n=(k+1)/\ell$. We build a sequence $\sigma$ that cycles through all $n\ell=k+1$ pages, respecting the layer constraints, achieving the worst case for \LRU compared to \OPT.

\noindent
\begin{minipage}[c]{0.4\linewidth}
$$
\begin{array}[c]{ccc}
  \sigma=(
    & p_1^{(1)}, p_1^{(2)}, \ldots, p_1^{(\ell)} & \\
    & p_2^{(1)}, p_2^{(2)}, \ldots, p_2^{(\ell)} & \\
    & \ldots&\\
    & p_{n}^{(1)}, p_{n}^{(2)}, \ldots, p_{n}^{(\ell)} & \\
    & p_1^{(1)}, p_1^{(2)}, \ldots, p_1^{(\ell)} & \\
    & \ldots & ).
\end{array}
$$
\ao{This sequence should fit the illustration on the right}
\end{minipage}%
\begin{minipage}[c]{0.6\linewidth}
  \scalebox{0.92}{%
\begin{tikzpicture}[scale=0.9, every node/.style={font=\footnotesize}]
    \definecolor{incache}{RGB}{135,206,250} %

    \pgfmathsetmacro{\rows}{3}
    \pgfmathsetmacro{\cols}{6}

    \tikzset{
        fillcell/.style={
            rectangle,
            minimum size=0.85cm,
            fill=#1!50,
            draw,
            draw=#1!80!black,
            rounded corners=2pt,
        },
        fillcell/.default=gray, %
    }

    \foreach \j in {1,...,\cols} {
        \node[above] at (\j, -0.4) {$L_\j$};
    }
    \foreach \i in {1,...,\rows} {
        \node[left] at (0.3, -\i) {Token $\i$};
    }

    \node[fillcell] at (1,-1) {1};
    \node[fillcell=incache] at (2,-1) {2};
    \node[fillcell=incache] at (3,-1) {3};
    \node[fillcell=incache] at (4,-1) {4};
    \node[fillcell=incache] at (5,-1) {5};
    \node[fillcell=incache] at (6,-1) {6};
    \node[fillcell=incache]        at (1,-2) {7};
    \node[fillcell=incache] at (2,-2) {8};
    \node[fillcell=incache] at (3,-2) {9};
    \node[fillcell=incache] at (4,-2) {10};
    \node[fillcell=incache] at (5,-2) {11};
    \node[fillcell=incache] at (6,-2) {12};
    \draw[very thick, red] (0.6, -2.6) rectangle ++(0.8, -0.8) node[pos=.5] {1};
    \node at (2, -3) {2};
    \node at (3, -3) {3};
    \node at (4, -3) {4};
    \node at (5, -3) {5};
    \node at (6, -3) {6};
    
\end{tikzpicture}
}
  
  \centerline{Illustration with $n=2$ and $\ell=6$.}
\end{minipage}
\bigskip

We construct a request sequence $\sigma$ such that the $i$-th round subsequence equals $(p_{i'}^{(1)}, p_{i'}^{(2)}, \dots, p_{i'}^{(\ell)})$ with $i' = i \mod n$. That is, the request repeats itself cyclically with a period of $n$ rounds. We analyze the sequence in chunks of $n$ rounds, meaning that within each such chunk, we observe $n \ell$ total requests.

Since there are $k + 1$ distinct pages, the \LRU policy will always evict a page just before it is needed again. This means that every request results in a cache miss. \OPT evicts the last page requested at each cache miss, so suffers a cache miss every $k$ requests.

Therefore, $\CR(\LRU) \geq k$.
\qed
\end{proof}
\fi

\subsection{Randomized Algorithms}
\label{sec.rand}

In this section we turn our attention to randomized algorithms. Recall that in  standard online paging, randomization can help improving the competitive ratio to $H_k=\Theta(\log k)$, which is tight, as discussed in Section~\ref{subsec:related}. The following theorem is the main result of this section, which establishes a lower bound on the competitive ratio for the layered paging problem.

\begin{theorem}
  \label{th.rand-competitive-LB}
  Every randomized algorithm for $\ell$-layered paging with $n$ experts has 
  competitive ratio at least $\max(H_n, \frac{\log(\ell)}{6n})$.
\end{theorem}

Before we present the proof of Theorem~\ref{th.rand-competitive-LB}, let us discuss its significance. First, note that layered paging becomes trivial from the point of view of competitive analysis if $k\geq n\ell$: this is because the algorithm can store all pages in the cache, and thus achieve a competitive ratio equal to 1. We can thus assume that $k \leq n\ell$. Hence, if $n \in O(1)$, then Theorem~\ref{th.rand-competitive-LB} shows a lower bound $\Omega(\log \ell)=\Omega(\log k)$. Similarly, if $\ell \in O(1)$, then the theorem shows a lower bound of $\Omega(\log n)=\Omega(\log k)$. Therefore, the result shows that if either $n$ or $\ell$ is constant (as typically in practice), then any randomized algorithm has competitive ratio at least $\Omega(\log k)$, which is tight. 

To prove Theorem~\ref{th.rand-competitive-LB}, we will rely on a variant of the coupon collector problem known as the {\em parallel} coupon collector problem~\cite{ferrante_coupon-collectors_2016}. In this problem, there are $N$ different types of coupons, and $C$ agents (collectors). In each round, each of the $C$ agents selects one of the $N$ coupon types uniformly at random. Agents are not allowed to exchange or share their coupons. Define $T(N,C)$ as the number of rounds after which each agent has collected all $N$ coupons. Note that this problem is a generalization of the well-known vanilla coupon collector problem in which $C=1$~\cite{blom1993problems}, which is a key part in the analysis of the standard online paging problem~\cite{motwani_randomized_1995}. We are interested in bounding from below the {\em cover time}, namely the quantity $\mathbb{E}(T(N,C))$. We accomplish this in the following theorem, as previous work did not establish usable bounds~\cite{ferrante_coupon-collectors_2016}.

\begin{theorem}
For the parallel coupon collector with $N\geq 2$ coupons and $C$ agents, its cover time is such that
\[
\mathbb{E}(T(N,C)) \geq \max\{N \cdot H_N, \frac{\log C}{6}\}.
\]
\label{thm:coupon.combined}
\end{theorem}

\begin{proof}

We prove the lower bounds separately, starting by $\mathbb{E}(T(N,C)) \geq N\cdot H_N$.

    We focus of  the first collector, corresponding to the first layer. We know from the traditional
    coupon collector problem that it will finish on average after $N H_N$
    rounds.  By definition,
     all the collectors finish collecting their coupons after the first collector finishes its
    collection, hence the first bound:
    $\mathbb{E}(T(N, C)) \geq \mathbb{E}(T(N, 1)) = N H_N.$

We now show that $\mathbb{E}(T(N,C)) \geq \frac{\log(C)}{6}$ for $N\geq 2$.

\DeclarePairedDelimiter\p{(}{)}
  We first compute the probability of the expected time to complete
  all collections to be larger than $t$. By definition:
  $$
  \mathbb{P}(T(N, C) \geq t) = 1 - \mathbb{P}(T(N, C) < t) = 1 - \prod_{i=1}^\ell \mathbb{P}(T(N, 1) < t).
  $$
  Considering only the probability of not getting the first coupon, we get:
  $$
  \mathbb{P}(T(N, C) \geq t)
   \geq 1 - \p*{1 - \p*{\frac{N-1}{N}}^t}^C
   \geq 1 - \exp(-C\p*{\frac{N-1}{N}}^t).
   $$
   For $ t = {\log(\frac{C}{\log(2)})} ~/~ {\log(\frac{N}{N-1})}$, we have $\mathbb{P}(T\big(N, C) \geq t\big)\geq \frac{1}{2}$. 

   Thus $    \mathbb{E}(T(N, C)) \geq {\log(\frac{C}{\log(2)})} ~/~  {2 \log(\frac{N}{N-1})} \geq \frac{\log(C)}{6}.$
\ifreport

With more details: Using $t \log(\frac{N}{N-1}) = \log(\frac{C}{\log(2)})$ gives
\begin{align*}
  t \log(\frac{N-1}{N}) &= \log(\frac{\log(2)}{C}) \\
  \p*{\frac{N-1}{N}}^t &= \frac{\log(2)}{C} \\
  -C \p*{\frac{N-1}{N}}^t &= \log(\frac{1}{2}) \\
  \exp(-C \p*{\frac{N-1}{N}}^t) &= \frac{1}{2} \\
  1 - \exp(-C \p*{\frac{N-1}{N}}^t) &= \frac{1}{2} 
                                         \intertext{Thus:}
                                         \mathbb{P}\p*{T(N, C) \geq \frac{\log(\frac{C}{\log(2)})}{\log(\frac{N}{N-1})}} &\geq \frac{1}{2} \\
  \mathbb{E}(T(N, C)) &\geq \frac{\log(\frac{C}{\log(2)})}{2 \log(\frac{N}{N1})} \\
  \mathbb{E}(T(N, C)) &\geq \frac{\log(\frac{C}{\log(2)})}{2 (1 + \frac{N}{N-1})} \\
  \mathbb{E}(T(N, C)) &\geq \frac{N-1}{4N - 2} \log(\frac{C}{\log(2)}) \\  %
  \mathbb{E}(T(N, C)) &\geq \frac{1}{6} \log(\frac{C}{\log(2)}) \\
  \mathbb{E}(T(N, C)) &\geq \frac{\log(C)}{6}
\end{align*}
\fi
\qed
\end{proof}

We can now proceed with the proof of Theorem~\ref{th.rand-competitive-LB}. It relies on showing that every randomized paging algorithm has a competitive ratio at least $\frac{\ell}{k} \cdot \mathbb{E}(T(n, \ell))$.

\begin{proof}[Theorem~\ref{th.rand-competitive-LB}]

We consider an instance of the layered paging problem such that $n\ell = k+1$.
    Let $A$ be a randomized algorithm. Using Yao's principle, we consider a
    random sequence $\sigma$ as input, such that the $i$-th requested page $\sigma_i$ is drawn uniformly at random
    among the $n$ pages of layer $i \mod \ell$, while $A$ is considered
    deterministic. 

    At the start of a round, the deterministic algorithm has all pages but one in cache.
    There is at least a $\frac{1}{n}$ chance of having a cache miss per
    round, the probability of drawing the missing page. There are $\frac{|\sigma|}{\ell}$ rounds, thus we have:

    $$\mathbb{E}(A(\sigma)) \geq \frac{|\sigma|}{\ell n}.$$ %

    For the analysis of the optimal algorithm \OPT, we partition the random
    sequence $\sigma$ into blocks such that each block is a minimal contiguous sub-sequence containing $k+1$
    distinct elements. \OPT must suffer at least one cache miss per block, as its cache is missing one requested page at the start of each block.
    It remains to estimate the size of the blocks. There is an equivalence with the parallel coupon collector problem, with $n$
    coupons and $\ell$ collectors, associating a round of $\ell$ requests to
    one parallel step of all collectors. The average length of a block is then $\mathbb{E}(T(n,\ell))$ rounds of $\ell$ requests.
    For a sequence $\sigma$ long enough compared to $n$ and $\ell$, we can estimate the number of blocks in the sequence to the inverse of the average length of a block times the length of the sequence. We therefore obtain $\mathbb{E}(\OPT(\sigma)) \leq \frac{|\sigma|}{\ell \mathbb{E}(T(n,\ell))}$. We conclude:

    $$\mathbb{E}(A(\sigma)) \geq \frac{|\sigma|}{\ell n} \geq \mathbb E(\OPT(\sigma)) \frac{\ell
    \mathbb{E}(T(n,\ell))}{\ell n}.$$

    Thus $\CR(A) \geq \frac{ \mathbb{E}(T(n,\ell))}{n}$, for a long enough $\sigma$.

Theorem~\ref{th.rand-competitive-LB} then follows from Theorem~\ref{thm:coupon.combined} :
$$
\CR(A) 
        \geq \frac{\mathbb{E}(T(n,\ell))}{n} 
        \geq \max\left(H_n, \frac{\log(\ell)}{6n}\right). %
$$
\qed
\end{proof}

Contrarily to the classical paging problem, it is still an open question
whether there exist randomized algorithms reaching this ratio, or simply
with a better competitive ratio that $H_k$.

\section{A Layered Extension of the \LRU Algorithm} 
\label{sec:llru}
In this section, we propose an extension of the \LRU algorithm adapted to the problem at hand. The algorithm handles better some pathological scenarios specific to this problem. More specifically, note that the standard \LRU policy evicts the oldest page in the cache, regardless of which
layer it belongs to. This is intuitively inefficient: if a page of layer $i$ is currently served, then it is generally unsafe to evict a page belonging to layer $i+1$, since it may be requested in the immediately following step. Instead, a more reasonable choice is to evict a page of layer $i-1$ or even $i$. We thus seek a {\em layer-specific} variant of \LRU, and to this end we define two quantities, specific to a page $p$ in  cache and a time $t$.  

\begin{definition}
Given a page $p$ in the cache at the current time $t$, let $\tau(p,t)$ denote the last time that $p$ was requested. We define the {\em last-round index} of $p$, denoted by $R(p,t)=\lfloor \frac{t -
    \tau(p,t)}{\ell} \rfloor$, as the number of rounds between
  $\tau(p,t)$ and $t$. We also define the {\em relative layer distance of $p$}, noted $D(p,t)= \tau(p,t) - t \mod \ell$, as the number of requests needed before reaching the layer in which $p$ belongs.
\label{def:llru}
\end{definition}

Using the above definitions, we define the {\em Layered Last Recently Used} caching algorithm (\LLRU), as the policy which, on a cache miss at time $t$, evicts the page $p$ with largest index $R(p,t)$, and in the event of ties evicts the page with largest index $D(p,t)$. This allows, intuitively, to evict pages which have not been requested since a large number of rounds, and whose layer will be requested the latest. Figure~\ref{fig:llru_step} illustrates an execution of \LLRU.

\begin{figure}[htbp]
    \centering
    \begin{subfigure}[t]{0.45\textwidth}
        \centering
        \scalebox{0.55}{%
\begin{tikzpicture}[scale=0.9, every node/.style={font=\footnotesize}]
    \definecolor{layer1}{RGB}{240,128,128} %
    \definecolor{layer2}{RGB}{135,206,250} %
    \definecolor{layer3}{RGB}{144,238,144} %

    \pgfmathsetmacro{\rows}{3}
    \pgfmathsetmacro{\cols}{6}

    \tikzset{
        fillcell/.style={
            rectangle,
            minimum size=0.85cm,
            fill=#1!50,
            draw,
            draw=#1!80!black,
            rounded corners=2pt,
        },
        fillcell/.default=gray, %
    }

    \foreach \j in {1,...,\cols} {
        \node[above] at (\j, -0.4) {$L_\j$};
    }
    \foreach \i in {1,...,\rows} {
        \node[left] at (0.3, -\i) {Token $\i$};
    }

    \ifnextstep
        \node[fillcell]        at (1,-1) {};
        \node[fillcell=layer1] at (2,-1) {2};
        \node[fillcell]        at (3,-1) {};
        \node[fillcell]        at (4,-1) {};
        \node[fillcell=layer1] at (5,-1) {1};
        \node[fillcell]        at (6,-1) {};
        \node[fillcell]        at (1,-2) {};
        \node[fillcell=layer2] at (2,-2) {6};
        \node[fillcell=layer2] at (3,-2) {5};
        \node[fillcell=layer2] at (4,-2) {4};
        \node[fillcell=layer2] at (5,-2) {3};
        \node[fillcell=layer3] at (6,-2) {11};
        \node[fillcell=layer3] at (1,-3) {10};
        \node[fillcell=layer3] at (2,-3) {9};
        \node[fillcell=layer3] at (3,-3) {8};
        \node[fillcell=layer3] at (4,-3) {7};
        \draw[very thick, red] (4.6, -2.6) rectangle ++(0.8, -0.8);
        \draw[very thick, red] (3.6, -0.6) -- (4.4, -1.4);
        \draw[very thick, red] (3.6, -1.4) -- (4.4, -0.6);
    \else
        \node[fillcell]        at (1,-1) {};
        \node[fillcell=layer1] at (2,-1) {2};
        \node[fillcell]        at (3,-1) {};
        \node[fillcell=layer1] at (4,-1) {1};
        \node[fillcell=layer2] at (5,-1) {6};
        \node[fillcell]        at (6,-1) {};
        \node[fillcell]        at (1,-2) {};
        \node[fillcell=layer2] at (2,-2) {5};
        \node[fillcell=layer2] at (3,-2) {4};
        \node[fillcell=layer2] at (4,-2) {3};
        \node[fillcell=layer3] at (5,-2) {11};
        \node[fillcell=layer3] at (6,-2) {10};
        \node[fillcell=layer3] at (1,-3) {9};
        \node[fillcell=layer3] at (2,-3) {8};
        \node[fillcell=layer3] at (3,-3) {7};
        \draw[very thick, red] (3.6, -2.6) rectangle ++(0.8, -0.8);
    \fi
\end{tikzpicture}
}
        \caption{\LLRU before processing layer $4$}
        \label{fig:llru_step_a}
    \end{subfigure}
    \hfill
    \begin{subfigure}[t]{0.45\textwidth}
        \centering
        \nextsteptrue
        \scalebox{0.55}{%
\begin{tikzpicture}[scale=0.9, every node/.style={font=\footnotesize}]
    \definecolor{layer1}{RGB}{240,128,128} %
    \definecolor{layer2}{RGB}{135,206,250} %
    \definecolor{layer3}{RGB}{144,238,144} %

    \pgfmathsetmacro{\rows}{3}
    \pgfmathsetmacro{\cols}{6}

    \tikzset{
        fillcell/.style={
            rectangle,
            minimum size=0.85cm,
            fill=#1!50,
            draw,
            draw=#1!80!black,
            rounded corners=2pt,
        },
        fillcell/.default=gray, %
    }

    \foreach \j in {1,...,\cols} {
        \node[above] at (\j, -0.4) {$L_\j$};
    }
    \foreach \i in {1,...,\rows} {
        \node[left] at (0.3, -\i) {Token $\i$};
    }

    \ifnextstep
        \node[fillcell]        at (1,-1) {};
        \node[fillcell=layer1] at (2,-1) {2};
        \node[fillcell]        at (3,-1) {};
        \node[fillcell]        at (4,-1) {};
        \node[fillcell=layer1] at (5,-1) {1};
        \node[fillcell]        at (6,-1) {};
        \node[fillcell]        at (1,-2) {};
        \node[fillcell=layer2] at (2,-2) {6};
        \node[fillcell=layer2] at (3,-2) {5};
        \node[fillcell=layer2] at (4,-2) {4};
        \node[fillcell=layer2] at (5,-2) {3};
        \node[fillcell=layer3] at (6,-2) {11};
        \node[fillcell=layer3] at (1,-3) {10};
        \node[fillcell=layer3] at (2,-3) {9};
        \node[fillcell=layer3] at (3,-3) {8};
        \node[fillcell=layer3] at (4,-3) {7};
        \draw[very thick, red] (4.6, -2.6) rectangle ++(0.8, -0.8);
        \draw[very thick, red] (3.6, -0.6) -- (4.4, -1.4);
        \draw[very thick, red] (3.6, -1.4) -- (4.4, -0.6);
    \else
        \node[fillcell]        at (1,-1) {};
        \node[fillcell=layer1] at (2,-1) {2};
        \node[fillcell]        at (3,-1) {};
        \node[fillcell=layer1] at (4,-1) {1};
        \node[fillcell=layer2] at (5,-1) {6};
        \node[fillcell]        at (6,-1) {};
        \node[fillcell]        at (1,-2) {};
        \node[fillcell=layer2] at (2,-2) {5};
        \node[fillcell=layer2] at (3,-2) {4};
        \node[fillcell=layer2] at (4,-2) {3};
        \node[fillcell=layer3] at (5,-2) {11};
        \node[fillcell=layer3] at (6,-2) {10};
        \node[fillcell=layer3] at (1,-3) {9};
        \node[fillcell=layer3] at (2,-3) {8};
        \node[fillcell=layer3] at (3,-3) {7};
        \draw[very thick, red] (3.6, -2.6) rectangle ++(0.8, -0.8);
    \fi
\end{tikzpicture}
}
        \caption{\LLRU after processing layer $4$}
        \label{fig:llru_step_b}
    \end{subfigure}
    \caption{
    One step of the \LLRU algorithm. The colored pages are the ones still in the cache. All pages that have the same last-round index are of the same color, with the number indicating the eviction
    priority given by the \LLRU algorithm.}
    \label{fig:llru_step}
    \figspace
\end{figure}
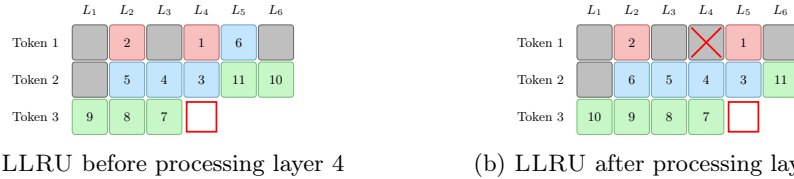

\section{Performance Evaluation through Simulations}
\label{sec:experiments}

\subsection{Datasets and Paging Strategies}

\paragraph{Datasets}

We use two different datasets in this section: traces coming from real MoE
usage, and synthetic datasets to study the sensitivity of caching
strategies to specific parameters.

The synthetic dataset is created by associating to each page a frequency of usage, sampled using a Zipf distribution. Such distributions have been proven to accurately model real-life cache systems \cite{kotera_modeling_2008,canon_hector_2023}. In a nutshell, a few pages will be frequently requested, while the other pages will rarely be requested. In practice, with two law parameters $a$ and $b$, the $j$-th page of a layer will have a probability of being sampled $p_j \sim \frac{1}{(j+b)^a}$.

Regarding MoE traces, we use two MoE LLMs: (a) the Mixtral 7B
model~\cite{jiang_mixtral_2024} (with $\ell=32$ layers and $n=8$ experts 
per layer) on $1000$ prompts from the VMWare Open Instruct
dataset\footnote{\url{https://huggingface.co/datasets/VMware/open-instruct}}
and (b) the Llama-MoE model~\cite{llama-moe} (with $n=16$ and
$\ell=32$) on $100$ prompts from the ``helpful instructions''
dataset~\cite{huggingface-helpful-instructions}. For each model, we
record the experts used for each token, for each layer. Since both
models use several experts for each layer (2 among 8 for Mixtral, 4
among 16 for Llama-MoE), while our model assumes a single expert used
for each layer, we create several rounds for the processing of a
single token: one round for the first expert, a second round for the
second expert, etc. For instance, for the Mixtral model, if the experts
used to produce one token are $(E_1^{(1)},E_2^{(1)})$ for the first layer, $(E_{1}^{(2)},E_{2}^{(2)})$ for the second layer, up to $(E_{1}^{(32)},E_{2}^{(32)})$, we replace this sequence by
the following two rounds: $E_1^{(1)},\ldots, E_{1}^{(32)}$ then
$E_2^{(1)},\ldots, E_{2}^{(32)}$.

\paragraph{Strategies} In addition to the \LLRU strategy described
above, we use strategies from the classical paging problem: \LRU as well as the \MARKER randomized algorithm~\cite{fiat_competitive_2002}. These
strategies can be applied either on the whole problem with the whole cache, without being aware
of the layers, or on each layer with a specific cache of size $k/\ell$
per layer. We denote this last variant by \textsc{Dist}: \LRUDIST,
\MARKERDIST, etc., to show that the cache is distributed among
layers. We also consider an optimal offline strategy \OPT (that knows
the whole sequence of requests) that follow Belady's rule \cite{belady_study_1966} for comparison.

\subsection{Results on MoE Traces}
\label{subsec:res_moe}

\begin{figure}[tb]
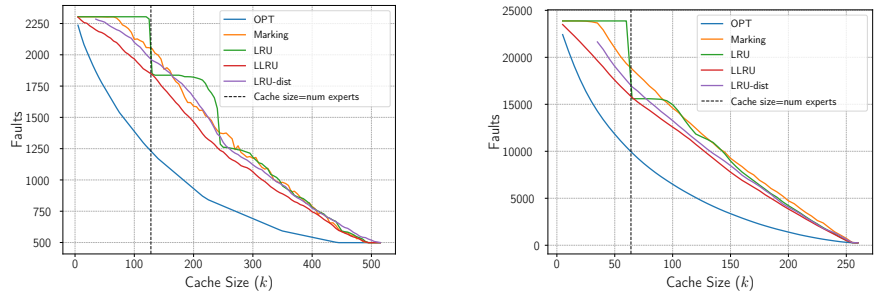

    \centering
    \begin{subfigure}{0.48\textwidth}
        \centering
        \resizebox{0.92\textwidth}{!}{
            \import{expe}{llama_0.pgf}
        }
        \caption{Cache performance on one Llama trace}
    \end{subfigure}
    \hfill
    \begin{subfigure}{0.48\textwidth}
        \centering
        \resizebox{0.92\textwidth}{!}{
            \import{expe}{mixtral_0.pgf}
        }
        \caption{Cache performance on one Mixtral trace}
    \end{subfigure}
    \caption{Comparison of \LRU, \LLRU, \MARKER and \OPT caching policies while varying cache size $k$}
    \label{fig:moe_one_trace}
    \figspace
\end{figure}

Our first experiments are direct comparisons of the previously described caching strategies on real MoE traces. In \Cref{fig:moe_one_trace}, we plot the number of cache misses of each algorithm on a single MoE trace with a cache size $k$ varying from $1$ (no cache hit) to $n \ell$ (no cache miss). As expected in both extreme cases, all the strategies fault equally as much as \OPT and the scheduling strategies are irrelevant. The intermediate values exhibit differences between strategies, and a sizable gap to \OPT that can go up to $\times 2.5$. The first observation from this figure is the great fluctuation in performance of \LRU depending on $k$. This well-known phenomenon is inherent to the cyclic behavior of \LRU which can evict pages right before they get requested again, similarly to its worst case against \OPT described in \Cref{sec.det}. This is the reason why real life applications use the \LRUDIST strategy instead which limits the fluctuations of performance, but we observe that this decision comes with an increased number of faults which we will further evaluate in the next section. In comparison, we observe that \LLRU gets the best of both worlds, achieving the lowest number of faults across online strategies, while keeping a very smooth progression with respect to the cache size. These observations are confirmed by \Cref{fig:moe_traces}, which represents the distribution of the normalized number of cache faults of each previously mentioned algorithm for ten MoE traces. Here, we take a cache of fixed size $200$ for both models to display two different "modes": Mixtral almost fits entirely (256 total experts) while LLama can't even store half (512 experts). The relative performance of algorithms is consistent with other cache sizes and in particular \LLRU significantly outperforms other policies. In this particular setting, \LLRU makes $\sim\!15\%$ less faults than \LRU and $\sim\!7\%$ less than \LLRUDIST for the LLama model, and is still $\sim\!4\!-\!5\%$ faster than \LRUDIST on the the Mixtral traces, though the choice of $k$ makes the instances very constrained. Furthermore, this figures demonstrates that these performance gains are pretty consistent, with a rather small interquartile spread.

\begin{figure}[tb]
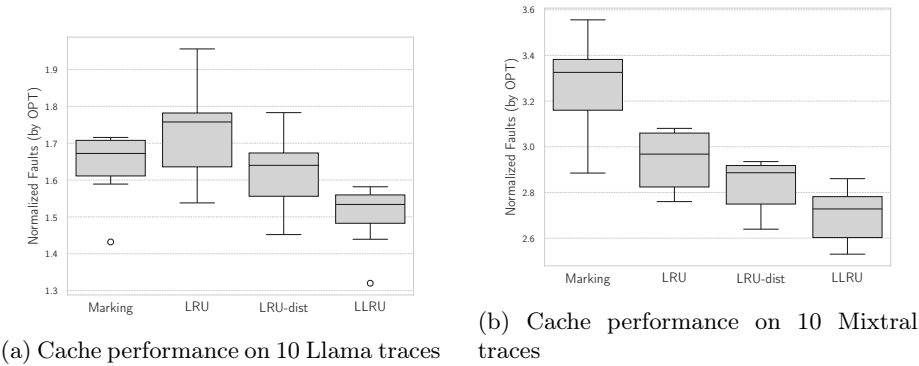

    \centering
    \begin{subfigure}{0.48\textwidth}
        \centering
        \resizebox{0.92\textwidth}{!}{
            \import{expe}{k=200,name=llama.pgf}
        }
        \caption{Cache performance on 10 Llama traces}
    \end{subfigure}
    \hfill
    \begin{subfigure}{0.48\textwidth}
        \centering
        \resizebox{0.92\textwidth}{!}{
            \import{expe}{k=200,name=mixtral.pgf}
        }
        \caption{Cache performance on 10 Mixtral traces}
    \end{subfigure}
    \caption{Comparison of normalized faults (by \OPT) for the \LRU, \LRUDIST, \LLRU and \MARKER caching policies with $k=200$}
    \label{fig:moe_traces}
    \figspace
\end{figure}

\subsection{Shared vs. Distributed Cache}

This second set of experiments aims at exploring the performance gap between shared and unified cache systems. We saw in \Cref{sec:competitive} that an algorithm that splits the cache can have an arbitrarily big competitive ratio, which we want to test in practice. In \Cref{fig:zipf_opt}, we compare the number of faults of \OPT and \OPTDIST using a synthetic trace generated with a Zipf distribution, for a range of values of $n$ and $\ell$ and a fixed cache of size $k$. We observe that for the values for which everything fits in the cache, the ratio is $1$, but whenever this is not possible anymore, $\OPTDIST$ makes significantly more cache faults, up to $3\times$ more than $\OPT$. Using some reasonable values for an MoE system ($32$ layers, $8$ experts per layer and $64$ experts fitting in the cache), we show in \Cref{fig:zipf_vary} how the $a$ parameter of the underlying Zipf law affects the performances. We observe two effects: when $a$ is very big, the data disparity is such that almost always the same one expert will be chosen at each layer, hence both \OPT and \OPTDIST achieve similar results. Whenever $a$ gets smaller, page disparity between layers is still present, and \OPTDIST starts underperforming, with up to $1.8\times$ more page faults. More importantly, whenever $a$ gets close to $0$, i.e. there is no rank imbalance, \OPTDIST still performs poorly. This further solidifies the theoretical findings of \Cref{sec:competitive} and real life experiments of \Cref{subsec:res_moe} which states that split caching is sub-optimal and should not be used in practice.

\begin{figure}[htbp]
    \centering
    \begin{subfigure}{0.48\textwidth}
        \centering
        \resizebox{0.92\textwidth}{!}{
            \import{expe/OPT}{gen=Zipf2.pgf}
        }
        \caption{Comparison of \OPT and \OPTDIST on a Zipf distribution of parameter $2$ with $k=16$ and several values of $n$ and $\ell$}    
        \label{fig:zipf_opt}
    \end{subfigure}
    \hfill
    \begin{subfigure}{0.48\textwidth}
        \centering
        \resizebox{0.92\textwidth}{!}{
            \import{expe/Zipf}{OPT.pgf}
        }
        \caption{Fault ratio between \OPT and \OPTDIST on a Zipf distribution with $k=64$, $\ell=32$ and $n=8$ with a varying law parameter $a$}
        \label{fig:zipf_vary}
    \end{subfigure}
    \caption{Comparison of \OPT and \OPTDIST}
    \figspace
\end{figure}

\section{Conclusion}
\label{sec:conclusion}

We introduced and studied a new online paging problem, namely {\em layered paging}, as a formulation of cache management in the widely-used Mixture-of-Experts architecture of Large Language Systems. To our knowledge, this is the first study of competitive analysis in the context of LLM caching. %
We established several lower bounds on the competitive ratio that characterize the limitations of both deterministic and randomized algorithms, and which are near-tight if either the number of layers or the number of experts is a small constant, as typical in practice. We also proposed a simple, yet efficient adaptation of the Least-Recently-Used paging strategy that is tailored to LLMs/MoEs, while retaining the worst-case guarantees of the standard LRU. We experimentally demonstrated, using both real and synthetic data, that this new strategy yields significant improvements over the standard LRU.

This work paves the way for future studies of this challenging problem. Notably, we would like to close the theoretical gap between the lower and the upper bounds which exists if $n$ and $\ell$ are allowed to have arbitrary and unrestricted growth. In addition, while our formulation assumes, for simplicity, that a single expert is used at each layer, it would be interesting to consider the more general model in which several experts are involved in each layer. Last, an interesting direction for future work is to consider {\em learning-augmented} variants, in which the algorithm can leverage a machine-learned prediction about future input items. Such prediction models have been extremely influential in the standard online paging~\cite{lykouris_competitive_2018}, and some simple predictors have already been proposed in the context of MoE/LLMs~\cite{mazur23}. %

\begin{credits}
  \subsubsection{\ackname}
  This work was partially funded  by the project PREDICTIONS, grant ANR-23-CE48-0010 from the French National Research Agency (ANR).

    \subsubsection{Artifact availability.}
    The artifact is available in the Zenodo repository\cite{simon_2025_15576758}.
    
  \subsubsection{\discintname} The authors have no competing interests to declare that are relevant to the content of this article.
\end{credits}

\bibliographystyle{splncs04}
\bibliography{selected.bib}

\end{document}